\documentclass[11pt,draftclsnofoot,onecolumn]{IEEEtran}
\usepackage{exscale,amsmath,amssymb,relsize,mathrsfs,epsfig,algorithmic}
\usepackage[boxed]{algorithm2e}
\usepackage{setspace}
\usepackage{soul,color}
\title{Dictionary Subselection Using an Overcomplete Joint Sparsity Model}
\author{Mehrdad Yaghoobi$^\dagger$, Laurent Daudet$^\ddagger$, Michael E. Davies$^\dagger$ \thanks{This work is supported by EU FP7, FET-Open grant number 225913 and EPSRC grants EP/J015180/1 and EP/K014277/1. MED acknowledges support of his position from the Scottish Funding Council and their support of the Joint Research Institute with the Heriot-Watt University as a component part of the Edinburgh Research Partnership in Engineering and Mathematics. LD acknowledges partial support from LABEX WIFI  (Laboratory of Excellence within the French Program "Investments for the Future") under references ANR-10-LABX-24 and ANR-10-IDEX-0001-02 PSL*.}\\
 \smaller{ $^\dagger$ Institute for Digital Communications (IDCom), The University of Edinburgh, EH9 3JL, UK }
 \\ { $^\ddagger$ Paris Diderot University / IUF,  Institut Langevin, 1, rue Jussieu 75005 Paris, France.}
\\ { \{ m.yaghoobi-vaighan, mike.davies\}@ed.ac.uk, laurent.daudet@espci.fr}
}

\def\x{\mathbf x}
\def\y{\mathbf y}
\def\X{\mathbf X}
\def\Y{\mathbf Y}
\def\R{\mathbb R}
\def\D{\mathbf D}
\def\n{\mathbf n}
\def\P{\mathcal{P}}
\newtheorem{theorem}{\textit{Theorem}}
\newtheorem{remark}{\textit{Remark}}
\newtheorem{lem}{\textit{Lemma}}
\newtheorem{proposition}{\textit{Proposition}}
\newtheorem{corollary}{\textit{Corollary}}
\newcommand{\The}{\mathbf{\Theta}}
\newcommand{\norm}[1]{\|#1\|}

\begin{document}

\maketitle
\begin{abstract}
Many natural signals exhibit a sparse representation, whenever a suitable describing model is given. Here, a linear generative model is considered, where many sparsity-based signal processing techniques rely on such a simplified model. As this model is often unknown for many classes of the signals, we need to select such a model based on the domain knowledge or using some exemplar signals. This paper presents a new exemplar based approach for the linear model (called the {\em dictionary}) selection, for such sparse inverse problems. The problem of dictionary selection, which has also been called the {\em dictionary learning} in this setting, is first reformulated as a joint sparsity model. The joint sparsity model here differs from the standard joint sparsity model as it considers an overcompleteness in the representation of each signal, within the range of selected subspaces. The new dictionary selection paradigm is examined with some synthetic and realistic simulations.
\end{abstract}

\section{Introduction}

The sparse signal model is one the most successful low-dimensional signal models for modern signal processing applications \cite{Baraniuk10}. In this model, any considered signal $\y \in \R^m$, can be represented as the sum of a few elementary functions, called the \textit{atoms}, plus some noise $\n \in \R^m$, as follows,
\begin{equation*}
 \y = \D \x + \n,
\end{equation*}
 where $\D \in \R^{m \times p}$, called the \textit{dictionary}, is the collection of the atoms and $\x \in \R^p$ is a sparse vector. In this setting, $\y$ is often called a sparse signal in $\D$. The additive noise is used to consider the inaccuracy of the measurement device or the model mismatch. While choosing an overcomplete dictionary, \textit{i.e.} $p > m$, gives us a flexibility to choose sparser representation, the extra redundancy can be damaging in ducking failures coefficient recovery. Therefore, the success of sparse signal models depends on how well we choose a redundant $\D$, which is the main focus of this paper. 

There is a lot of interest in building redundant dictionaries to make more flexible models and various techniques have already been proposed to design the dictionary using some domain knowledge, see for example \cite{Yaghoobi09e}, or learning the dictionary using a given set of exemplars \cite{Olshausen97}, see \cite{Rubinstein10} and \cite{Tosic11} for a more complete review on different dictionary selection techniques. The advantage of the first approach is the possibility of incorporating already known signal structures and often fast implementation of the dictionary. The second approach does not need such prior information about the signals, but they often find an unstructured dictionary with a computationally expensive implementation. We will combine these two methods in this paper, by considering a large set of potentially good atoms $\mathbf{\Phi} \in \R^{m \times n}, \ n > p$, called a {\em mother dictionary} here, and selecting a smaller set of atoms as the final dictionary $\D$. Fast implementation 
of such dictionaries are guaranteed, if the mother dictionary has such a property. For instance, scalar products of a given signal $x$ with a family of Gabor atoms of length $m$ can be implemented with a computational complexity of $\mathcal{O}(m \log m)$. Also, as we restrict the search space to the dictionaries with mother atoms, it can be learned using much less exemplars. In other words, restricting the dictionary to a subset of mother atoms, regularises the dictionary learning problem and avoids overfitting.

As all the atoms of $\D$ exist in $\mathbf{\Phi}$, any sparse signal in $\D$, can be represented using $\mathbf{\Phi}$. The reader may ask, if we can use the large dictionary $\mathbf{\Phi}$, why we need to shrink it to find a dictionary which at best can only sparsify the signal to the same level. The answer to this question can be given by noting that, finding the sparse approximations have non-polynomial complexity, in a general setting. The success of practical sparse approximation algorithms depends on some internal structures of the dictionary, including mutual coherence \cite{Donoho01}, Restricted Isometry Property (RIP) \cite{Candes05b} or the null-space property \cite{Gribonval03}. Dictionary size indirectly affects these properties such that larger dictionaries mostly make the sparse recovery more difficult. Roughly speaking, it is caused by the fact that putting more atoms in the dictionary, the atoms become more similar. Such similarities between different atoms, indeed make it more challenging 
to find which set of atoms represents the signals more accurately, \textit{i.e} the problem of exact (support) recovery. The approximation in such large dictionaries would also be noise sensitive, as small noise may cause the wrong atoms to become selected.
Finally, in coding, the cost of indexing which atoms being used in the representation ($\x$), \textit{a.k.a.} the binary significance map, grows as the dictionary size increases.
% This is the reason that we practically do not use a very large dictionary, despite the fact that the best dictionary is an infinitely large dictionary, including all possible atoms.

\subsection{Related Work}\label{sec:relatedworks}

The problem of dictionary design by combining the atoms of a mother dictionary was considered in \cite{Neff02,Yaghoobi09b,Rubinstein10b}. In this setting, an auxiliary sparse matrix combines the mother atoms, to generate a dictionary which fits the given learning samples. The size of dictionary is fixed here and as the learned dictionary is 
the multiplication of a sparse matrix and a structured matrix (with a possibly fast multiplication), we can implement such a dictionary in two steps, where each of them are cheaper than $\mathcal{O}(p^2)$. The dictionary selection problem can be interpreted as a particular case of sparse dictionary learning, when the sparse matrix can have only  $p < n$ non-zero elements, with one non-zero on each row. 

The problem of learning a dictionary, when the size of dictionary is not given, has been investigated in \cite{Yaghoobi09d}. The dictionary selection problem has also a similar approach, by finding smaller size dictionaries from the given larger reference dictionaries. The difference is that the reference dictionary is fixed throughout the learning here, which allows us to handle significantly larger problems and find computationally fast dictionaries.

The dictionary selection, which will be considered in this paper, is also related to the problem of subset selection in machine learning \cite{Krause10,Das11}, where the goal is to select the most relevant subset, which describes the whole set. \cite{Krause10} uses the fact that such a model selection can be formulated as a submodular cost minimisation. For such a formulation, there exist some canonical solvers, which guarantee to find a neighbourhood solution. The derived neighbourhood is indeed not small, which motivated Das and Kempe \cite{Das11} to present an alternative submodular formulation to reduce the approximation error.

\subsection{Contributions}

We here choose a different path to the mentioned dictionary selection techniques in previous section, by reformulating the problem as a generalised form of joint sparse representation problem \cite{Malioutov03,Cotter05}. To the authors' knowledge, it is the first time that the dictionary selection problem is modelled in this way. In this model, representation of each signal is not only $p$-joint sparse, it is also $k$-sparse in the selected joint sparsity support. We here assume $p>m$, which makes the representation of each signal in the selected $p$-joint support, non unique, where $k$-sparsity constraint can help to find the correct representation. 

Based on the new signal model, we need to solve a quadratic objective. As the signal model and the objective include unbounded solutions, we need to investigate the conditions that the problem is \textit{well-defined}. Such an analysis is useful for the convergence study of any algorithm introduces to solve the problem. The boundedness and uniqueness of the solutions of the introduced optimisation problem are also introduced in this paper.

As the dictionary can be found using the active rows of the coefficient matrix of the introduced optimisation program, we need to practically solve a non-polynomial time complexity problem. We here introduce a technique, which is inspired from the iterative hard thresholding for sparse approximations \cite{Herrity06,Blumensath10}, to find such an active set of atoms. The algorithm is equipped with a line-search technique to guarantee the monotonic decrease of the (positive) objective. The convergence of the algorithm is also investigated in this paper. 

In the numerical tests provided in this paper, the new approach is shown to recover the exact dictionary, in a large range of sparsity/overcompleteness parameters.

% Although deriving a theoretical guarantee for the exact dictionary recovery, is left for the future work, we practically show that the ideal dictionary can be recovered, when the signals are sparse and the set of given atoms is significantly larger than the ideal dictionary.

\subsection{Paper Organisation}

We initially formulate the dictionary selection problem as an overcomplete joint sparse representation problem in Section \ref{sec:mathematicalmodelling}. We then introduce an iterative algorithm to solve the problem approximately in Section \ref{sec:dicoselection} and show some dictionary recovery results with synthetic data simulation in Section \ref{sec:simulations}. We also show some simulation results on the Curvelet sub-dictionary selection for the finger print data in this section. The paper will be concluded in Section \ref{sec:conclusion}.

\section{Mathematical Modeling} \label{sec:mathematicalmodelling}

Let $\mathbf{Y} = \left[\mathbf{y}_l\right]_{l \in [1,L]}$ be a matrix made by training samples $\mathbf{y}_l \in \mathbb{R}^{m}$ and $\mathbf{\Phi} = [\mathbf{\phi}_i]_{i \in \mathcal{I}}, |\mathcal{I}| = n$ be a mother dictionary of normalised atoms $\mathbf{\phi}_i \in \mathbb{R}^{m}$. We assume that the generative dictionary $\mathbf{D} \in \mathbb{R}^{m \times p}, \ m \le p$ is made by a subset selection of atoms in $\mathbf{\Phi}$, \textit{i.e.} $\mathbf{D} = \left[\mathbf{\phi}_i\right]_{i \in \mathcal{J}}$ where $\mathcal{J} \subset \mathcal{I}$ and $|\mathcal{J}| = p < n$. We assume that each $\mathbf{y}_l$ is {\em approximately} generated by a $k$-sparse coefficient vector $\mathbf{\gamma}_l$,
\begin{equation*}
\mathbf{y}_l \approx \mathbf{D} \mathbf{\gamma}_l.
\end{equation*}
We want to find a dictionary that fulfils the two (apparently contradictory) objectives : few elements in the dictionary, and sparsest decomposition for each signal. In other words $\D$ which is both {\em small} and {\em efficient}!
The problem of \textit{optimal dictionary selection} can thus be defined as finding the index set $\mathcal{J}$ meeting those criteria, given $\mathbf{Y}$, $\mathbf{\Phi}$, $p$ and $k$. Let $\mathbf{X} \in \mathbb{R}^{n \times L}$ be a coefficient matrix and $f_{\mathsmaller{\mathcal{J}}}(i): [1,p] \mapsto [1,n]$ be the mapping that assigns the corresponding atom index of $\mathbf{\Phi}$ to the $i^{th}$ component of $\gamma_l$. By assigning $\{{\mathbf{x}_l}\}_\mathsmaller{f_{\mathsmaller{\mathcal{J}}}(i)} \gets \{{\mathbf{\gamma}_l}\}_{i}, \ \forall i \in [1,p], \ \forall l \in [1,L]$, while the other elements of $\mathbf{X}$ are set to zero, the generative model can be reformulated as,

\begin{equation}\label{eq:genmodel}
 \mathbf{Y} \approx \mathbf{\Phi} \mathbf{X}.
\end{equation}
As $\mathbf{X}$ is $k$-sparse in each column and $p$-row-sparse, \text{i.e.} only $p$ rows of $\mathbf{X}$ have non-zero components, it lies in the intersection of the following sets,

\begin{equation} \label{eq:ksparse}
 \mathcal{K} := \left\{\The \in \mathbb{R}^{n \times L}: \ \|\theta_l\|_0 \le k, \forall l \in [1,L]\right\}
\end{equation}

\begin{equation} \label{eq:nsparse}
\mathcal{P} := \left \{\The \in \mathbb{R}^{n \times L}: \ \|\The\|_{0,\infty} \le p\right\}
\end{equation}
where $\|\The\|_{0,\infty} = \|\nu\|_0$, with $\{\nu\}_i := \|\theta^\mathsmaller{(i)}\|_{\infty}$ and $\theta^\mathsmaller{(i)}$ is the $i^{th}$ row of $\Theta$. In other words, sets $\mathcal{K}$ and $\mathcal{P}$ are the sets of $n$ by $L$ matrices which respectively have $k$ non-zero elements on each column and $p$ non zero rows. The signals which can be represented using a coefficient matrix in $\mathcal{K} \cap \mathcal{P}$, the $(k,p)$-(overcomplete) joint sparse signals or, we will simply say that they follow the $(k,p)$-(overcomplete) joint sparsity model. 
Here, we actually combine the coefficient matrix and the dictionary parameters, \textit{i.e.} the index set of optimal dictionary, in a single matrix $\X$, where the optimal atom indices are specified by the locations of non-zero rows of $\X$.

The optimal dictionary $\mathbf{D}$, which can alternatively be indicated by $\mathcal{J}$, is defined as the solution of the following problem,

\begin{equation}\label{eq:odsformulation}
 \min_{\The} \|\mathbf{Y} - \mathbf{\Phi} \The \|_{F}^2, \ \operatorname{s.t.} \ \The \in \mathcal{K} \cap \mathcal{P}.
\end{equation}
$\D$ can actually be found using the solution of (\ref{eq:odsformulation}), by selecting the atoms of $\mathbf{\Phi}$ which have been used at least once in the representation of $\Y$. 
This formulation has some similarities with the convex formulation of Friedman \textit{et al.} \cite{Friedman10}, where they combine the convex $\ell_1$ and $\ell_2$ penalties to promote an overcomplete joint sparsity model. The alternative formulation (\ref{eq:odsformulation}), used in this study, has the benefit of being directly related to the size-$p$ dictionary selection problem. Furthermore, the associated iterative algorithm, as presented in Section \ref{sec:dicoselection} offers a complexity that scales well with the dimension of the problem, that can be large in many practical problems.

\subsection{Boundedness and Uniqueness of the Solutions}

The constraint set $\mathcal{K} \cap \mathcal{P}$ is unbounded. This means that for any given finite value $t$, there exists at least a point $\mathbf{X} \in \mathcal{K} \cap \mathcal{P}$ such that $\max_i \norm{\X_i}_\infty > t$. It is necessary to find a condition which guarantees the boundedness of the solution of (\ref{eq:odsformulation}). Such a condition is given in Lemma \ref{lem:bounded}. To prove this lemma, we use the following proposition.

\begin{proposition}\label{prop:mindistance}
 Let $\mathcal{B}_{r}^\infty$ be an open ball centred at the origin, with the radius $r$, defined by $\mathcal{B}_{r}^\infty = \{\mathbf{A} \in \R^{m \times L}, \max_{i,j}|\mathbf{A}_{i,j}| < r\}$. For a given $\zeta \in \R_+$, if 
\begin{equation}\label{eq:boundedcondition}
\operatorname{Null}(\mathbf{\Phi}) \cap \mathcal{K} \cap \mathcal{P} = \{\mathbf{0}\},
\end{equation}
there exists a finite radius $r \in \R_+$ such that, $\forall \ \X_{KP} \in (\mathcal{K} \cap \mathcal{P}) \setminus \mathcal{B}_{r}^\infty$,
\begin{equation*}
 \min_{\X_N \in \operatorname{Null}(\mathbf{\Phi})} \norm{\X_{KP} - \X_N}_F > \zeta,
\end{equation*}

\end{proposition}

\begin{proof}
$\operatorname{Null}(\mathbf{\Phi})$ is a (linear) subspace of $\R^{m \times L}$ and $\mathcal{K} \cap \mathcal{P}$ is a union of subspaces \cite{Lu08}, which intersect at the origin. The shortest distance between a given non-zero point $\X_{KP}$ in $\mathcal{K} \cap \mathcal{P}$ and $\operatorname{Null}(\mathbf{\Phi})$ is non-zero, as $\mathbf{0}$ is the only point in $\operatorname{Null}(\mathbf{\Phi}) \cap \mathcal{K} \cap \mathcal{P}$. This distance becomes larger, if $\X_{KP}$ moves away from the origin with the formula $\alpha \, \X_{KP}$, for $\alpha > 1$. Therefore, there exists a radius $r$, which any point in $\mathcal{K} \cap \mathcal{P}$, located outside of $\mathcal{B}_{r}^\infty$, is at least $\zeta$ away from the closest point in $\operatorname{Null}(\mathbf{\Phi})$. 
\end{proof}

\begin{lem}\label{lem:bounded}
 Let the null space of the operator $\mathbf{\Phi}$, in the space $\R^{m \times L}$, be noted by $\mathcal{N}$. The solutions of (\ref{eq:odsformulation}) are bounded if and only if $\mathcal{N} \cap \mathcal{K} \cap \mathcal{P} = \{\mathbf{0}\} $.
\end{lem}
\begin{proof}
 Let $\X$ be a solution of (\ref{eq:odsformulation}) and $\X_\mathcal{N}$ and $\X_{\mathcal{R}}$ respectively be the projection of $\X$ onto the null-space and range of $\mathbf{\Phi}$. As $\|\X\|_{F}^2 = \|\X_\mathcal{N}\|_{F}^2 + \|\X_{\mathcal{R}}\|_{F}^2$, we only need to show that $ \|\X_\mathcal{N}\|_F$ and $ \|\X_{\mathcal{R}}\|_F$ are bounded for any solution of (\ref{eq:odsformulation})\footnote{We here show that Frobenius norm of $\X$ is bounded, which induces the boundedness of $\max_i \norm{X_i}_\infty$.}. As the matrix $\mathbf{0} \in\mathcal{K} \cap \mathcal{P}$, any solution of (\ref{eq:odsformulation}) should then have smaller objective than this matrix. We can then have,
\begin{equation*}
 \begin{split}
  2 \|\Y\|_{F} 	&\ge \|\Y\|_{F} + \|\Y - \mathbf{\Phi} \X\|_{F} \\
		&\ge \|\mathbf{\Phi} \X \|_F \\
		&\ge \sigma_{min} \|\X_{\mathcal{R}}\|_F,
 \end{split}
\end{equation*}
where $\sigma_{min}$ is the minimum (non-zero) singular value of $\mathbf{\Phi}$. This induces $\|\X_{\mathcal{R}}\|_F \le 2 \sigma_{min}^{-1} \|\Y\|_F$, which is the boundedness of $\|\X_{\mathcal{R}}\|_F$. 

We respectively denote $\Lambda$ and $\bar{\Lambda}$ as the support index of $\X$, \textit{i.e.} $\X_{\lambda} \ne 0, \ \lambda \in \Lambda$, and its complement. The matrix $\mathbf{A}_{\Lambda}$ (respectively $\mathbf{A}_{\bar{\Lambda}}$) is a matrix which is equal to $\mathbf{A}$ on the support index (respectively on the complement of support index) and zero on the other indices. The solution $\X$ is zero on the indices specified by $\bar{\Lambda}$, \textit{i.e.} $\X_{\bar{\Lambda}} = \mathbf{0}$. $\X_{\bar{\Lambda}} = \X_{\mathcal{N}_{\bar{\Lambda}}} + \X_{\mathcal{R}_{\bar{\Lambda}}} = \mathbf{0}$, shows that $\X_{\mathcal{N}_{\bar{\Lambda}}} = - \X_{\mathcal{R}_{\bar{\Lambda}}}$. On the other hand, 
\begin{equation*}
\begin{split}
 \|\X_{\mathcal{R}_{\bar{\Lambda}}} \|_{F}^2 	&= \|\X_{\mathcal{R}}\|_{F}^2 - \| \X_{\mathcal{R}_{\Lambda}} \|_{F}^2 \\
						&\le 4 \sigma_{min}^{-2} \|\Y\|_{F}^2,
\end{split}
\end{equation*}
which assures the boundedness of $\X_{\mathcal{N}_{\bar{\Lambda}}}$. We finally need to show that $\X_{\mathcal{N}_{\Lambda}}$ is also bounded. Momentarily assume that $\X_{\mathcal{N}_{\Lambda}}$ is unbounded. $\X_{\Lambda} = \X_{\mathcal{R}_{\Lambda}} + \X_{\mathcal{N}_{\Lambda}}$ is in $\mathcal{K} \cap \mathcal{P}$ and $\X_{\mathcal{N}} \in \mathcal{N}$. As $\X_{\Lambda}$ is unbounded when $\X_{\mathcal{N}_{\Lambda}}$ is unbounded, we can use Proposition \ref{prop:mindistance} with $\zeta = \left( \norm{\X_{\mathcal{R}_{\Lambda}}}_{F}^2 + \norm{\X_{\mathcal{N}_{\bar{\Lambda}}}}_{F}^2\right)^\frac{1}{2}$ as follows,
\begin{equation*}
\begin{split}
% \norm{\X_{\mathcal{R}_{\Lambda}}}_F + \norm{\X_{\mathcal{N}_{\bar{\Lambda}}}}_F	=& \ \zeta \\
\zeta^2 	<& \ \norm{\X_{\Lambda} - \X_{\mathcal{N}}}_{F}^2 \\
	=& \ \norm{\left(\X_{\mathcal{R}_{\Lambda}} + \X_{\mathcal{N}_{\Lambda}}\right) - \left(\X_{\mathcal{N}_{\bar{\Lambda}}} + \X_{\mathcal{N}_{\Lambda}} \right) }_{F}^2\\
	=& \ \norm{\X_{\mathcal{R}_{\Lambda}} - \X_{\mathcal{N}_{\bar{\Lambda}}}}_{F}^2 \\
	=& \ \norm{\X_{\mathcal{R}_{\Lambda}}}_{F}^2 + \norm{\X_{\mathcal{N}_{\bar{\Lambda}}}}_{F}^2,
\end{split}
\end{equation*}
which contradicts with the fact that $\zeta^2 = \norm{\X_{\mathcal{R}_{\Lambda}}}_{F}^2 + \norm{\X_{\mathcal{N}_{\bar{\Lambda}}}}_{F}^2$. Therefore the assumption of unboundedness of $\X_{\mathcal{N}_{\Lambda}}$ is incorrect, which complete the proof of boundedness of $\X$.

% Showing that, if $\mathcal{N} \cap \left\{ \X : \X \in \mathcal{K} \cap \mathcal{P} \right\}$ has a non trivial element, we can have unbounded solutions, is easy. 
If (\ref{eq:boundedcondition}) is not valid, we have a non-zero $\Delta \in \R^{m \times L}$ in the null space of $\mathbf{\Phi}$, which is also overcomplete joint sparse. This means that any non-zero $(k,p)$-joint sparse solution $\X$, with the same support as $\Delta$, can generate another solution of (\ref{eq:odsformulation}) by $\X + \lambda \Delta, \ \lambda \in \R$. By tending $\lambda$ to infinity, such a solution would be unbounded, which shows the necessity of $\mathcal{N} \cap \mathcal{K} \cap \mathcal{P} = \{\mathbf{0}\}$.
\end{proof}

It is generally difficult to check (\ref{eq:boundedcondition}) for a given mother dictionary. However, if the mother dictionary is in a general position, when the dimension of signal space $nL$ is larger than the sum of the dimension of null space $L(n-m)$ and each subspace $kL$, which means $k<m$, the Lebesgue measure of the lhs of (\ref{eq:boundedcondition}) is zero.

Although this lemma shows the boundedness of the solutions, it does not provide any explicit bound for the results. It means that if the $\operatorname{Null}(\mathbf{\Phi})$ subspace is very close to one of the subspaces in $\mathcal{K} \cap \mathcal{P}$, $\zeta$ can become very large. 

The reader may noticed that we did not use the {\em optimality} of $\X$, in the proof of Lemma \ref{lem:bounded}. Instead, we used the fact that the objective at $\X$ is less than the objective at $\mathbf{0}$. Therefore we can easily extend this lemma, to derive the boundedness of the search space.

\begin{corollary}\label{corol:bounded}
 The set $\{\The \in \mathcal{K} \cap \mathcal{P}, \ \norm{\Y-\mathbf{\Phi}\The}_F \le  \norm{\Y}_F\}$ is bounded if (\ref{eq:boundedcondition}) is true.
\end{corollary}

It is always useful to know when an optimisation problem like (\ref{eq:odsformulation}), has a unique solution. This is particularly useful in the dictionary design problem, as the other formulations has often multiple solutions. This is caused by the fact that any permutation of a dictionary is also a solution for the problem. This indeed makes the convexification of the problem much more challenging.

We can use a general theorem of the Union of Subspaces (UoS) model to show the injection of the mapping $\Phi$. \cite[Theorem 2.6]{Blumensath09a}, with our settings,  shows that if $k<m$, almost all linear maps $\Phi_d = \operatorname{diag}\{\Phi\} \in \R^{mL \times nL}, \ \Phi \in \R^{m \times n}$ are one to one on almost all elements of the $(k,p)$-joint sparse matrices. $\Phi_d$ is a diagonal matrix with $\Phi$ on the main diagonal. Interested readers may notice that the derived condition, \textit{i.e.} $k<m$, is indeed the sufficient condition for the lhs of (\ref{eq:boundedcondition}) to have zero measure.

% With such a guarantee, if there exists a solution for the optimal dictionary selection problem, it is the only admissible solution. This is interesting as we often have non-unique solutions in the dictionary learning problems \cite{Yaghoobi08a}, which makes the convexification of the problem quite challenging. 

We now derive a sufficient condition for the uniqueness of the solution in a deterministic sense. It is indeed a particular case of the uniqueness results for the UoS model \cite{Lu08}.

\begin{lem}\label{lem:uniqueness}
 Let $k \le \frac{m}{2}, \ p \le \frac{n}{2}$, $\mathcal{N} = \operatorname{Null}\{\mathbf{\Phi}\}$ and $\mathcal{K}_{2k} := \left\{\mathbf{X} \in \mathbb{C}^{m \times L}: \ \|\mathbf{x}_l\|_0 \le 2k, \forall l \in [1,L]\right\}$ and $\mathcal{P}_{2p} := \left \{\mathbf{X} \in \mathbb{C}^{m \times L}: \ \|\mathbf{X}\|_{0,\infty} \le 2p\right\}$ The optimisation problem (\ref{eq:odsformulation}) has a unique solution if 
\begin{equation}\label{eq:uniquenesscond}
\mathcal{N} \cap \mathcal{K}_{2k} \cap \mathcal{P}_{2p} = \{\mathbf{0}\}  
\end{equation}
\end{lem}
\begin{proof}
Let the solution not be unique and we have $\X_1$ and $\X_2$ as two distinctive solutions of (\ref{eq:odsformulation}). We have $\mathbf{\Phi} \X_1 = \mathbf{\Phi} \X_1 = \Y$, which means $\mathbf{\Phi} (\X_1 - \X_2) = \mathbf{0}$. As $\X_1 - \X_2 \in \mathcal{K}_{2k} \cap \mathcal{P}_{2p}$ and  $\X_1 - \X_2 \in \mathcal{N}$, it should be $\mathbf{0}$, which gives $\X_1 = \X_2$ and it contradicts with the fact that they are distinctive solutions.
\end{proof}

\begin{remark}
 Note that Lemma \ref{lem:uniqueness} presents a {\em sufficient} condition for the uniqueness of the solution, which is different to the standard k-sparse and k-joint sparse UoS models. Similar to the general form of block-sparse model, this is caused by the fact that some of the sparsity patterns in $\mathcal{K}_{2k} \cap \mathcal{P}_{2p}$ can not be divided to two disjoint  sparsity patterns in  $\mathcal{K} \cap \mathcal{P}$.
\end{remark}

\begin{remark}
The boundedness of the solutions of (\ref{eq:odsformulation}) needs a weaker condition than its uniqueness. We can actually use the uniqueness condition of Lemma \ref{lem:uniqueness} to show the boundedness of the solution.
\end{remark}

\begin{figure}[t]
\centering
% \centerline{\epsfig{figure=.eps,width=9cm}}
\includegraphics[width=8.5 cm]{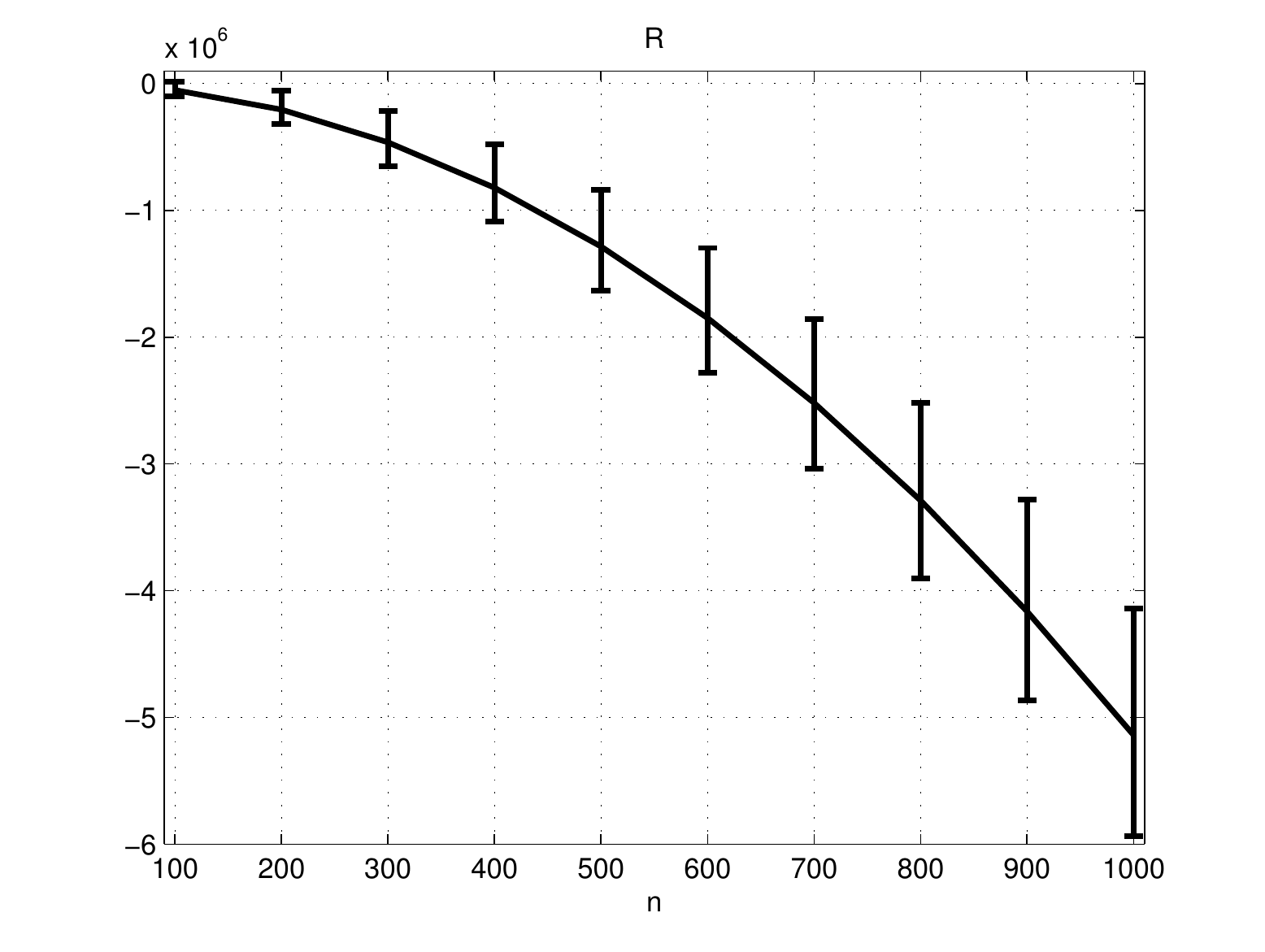}
\caption{$\mathcal{R}$ for different $n$, when $\delta = \frac{1}{4}$, $\rho = \frac{1}{10}$ and $t = 100$. $f$ is shown by the solid line and the bounds for $\mathcal{R}$ are shown with error bars, see (\ref{eq:R}).}
\label{numberofsubspaces}
\end{figure}

\subsection{Number of Subspaces}

It was mentioned that the introduced signal model is a UoS model, as fixing the support coefficient, generates a low-dimensional subspace of the $\R^{m \times L}$. We are restricting the set of matrices which are $k$-sparse on each column, to the matrices which are also $p$-joint sparse. Such a restriction reduces the number of admissible subspaces, which increases the robustness of the mapping $\Phi$ on its domains. In practical applications, we need some robustness to the noise and model mismatches for a successful sparse recovery. This is indirectly related to the distance between each two distinct points, after mapping. If two points have some small distance after the mapping, the embedding is sensitive to the noise. A measure which characterises such a robustness is the restricted isometry constant $\delta$ for each UoS model \cite{Candes05b}: a large $\delta$ ensures a more robust embedding. We refer the readers to \cite{Candes05b} for more information about the definition and implication of the restricted isometry constant.

Based upon \cite[Corollary 3.6]{Blumensath09a}, a necessary number of measurements to have a robust embedding with a particular $\delta$, has a lower bound, which is proportional to $\ln(N_s)$ and inversely proportional to $c_\mathsmaller{\delta} - \ln(\Delta_s)$, where $N_s$ is the total number of subspaces, $c_\mathsmaller{\delta}$ is a function of $\delta$ and $\Delta_s$ is the subspace separation of the proposed UoS \cite[Eq. (18)]{Blumensath09a}. $\Delta_s$ decreases by restricting the UoS to a subset of the original UoS. We therefore reduce the necessary number of training samples in this context, by decreasing $N_s$. In the following, we characterise the reduction in the number of subspaces, using the proposed UoS model, in the comparison with the $k$-sparse signal model. 

% which we quantify it here.
When the matrix is $k$-sparse on rows, we have $L$ times ${n \choose k}$ options to choose the support. The number of subspaces is thus ${n \choose k}^L$. If we also restrict the matrices to be $p$-joint sparse, we choose $k$ positions for each row, within the selected $p$ rows. We have therefore ${p \choose k}^L {n \choose p}$ subspaces. To quantify exponential reduction in the number of subspaces using the $(k,p)$-joint sparsity model, we approximate $\mathcal{R}$, defined as,
\begin{equation}\label{eq:subspacereduction}
\mathcal{R} := \log_2 \frac{{p \choose k}^L {n \choose p}}{{n \choose k}^L}.
\end{equation}
To find some upper and lower bounds for $\mathcal{R}$, we use the concept of {\em binary entropy} $\mathcal{H}$ from Information Theory, which is defined as follows,
\begin{equation}
 \mathcal{H}(\tau) \triangleq -\tau \log_2(\tau) - (1-\tau)\log_2(1-\tau) 
\end{equation}
where $0 \le \tau \le 1$ is the probability of a binary number. We can now bound ${n \choose k}$ as follows \cite[eq. 12.40]{Cover91},
\begin{equation}
 \frac{1}{n+1} 2^{n \mathcal{H}\left(\frac{k}{n}\right)} \le {n \choose k} \le 2^{n \mathcal{H}\left(\frac{k}{n}\right)}.
\end{equation}
Using the similar bounds for ${p \choose k}$ and ${n \choose p}$, and after some simple algebraic manipulations, we can derive a bound for $\mathcal{R}$ as follows,
\begin{equation}\label{eq:R}
\begin{split}
 -L \log_2 (p+1)  -& \log_2 (n+1) + f(k,p,n,L) \\ 
 &\le \mathcal{R} \le L\log_2 (n+1) + f(k,p,n,L),
\end{split}
\end{equation}
where $f(k,p,n,L) \triangleq n \mathcal{H}\left(\frac{p}{n}\right) + Lp\mathcal{H}\left(\frac{k}{p}\right) - nL\mathcal{H}\left(\frac{k}{n}\right)$
% We use Stirling's approximation, \textit{i.e.} $n! \approx \sqrt{2n\pi} n^n e^{-n}$, and ignore the $\mathcal{O}\{\sqrt{n}\}$ factor. This approximation is asymptotically accurate when $n \gg 1$. 
If we replace the binary entropy in $f(k,p,n,L)$, we can derive an explicit formulation for $f$ as follows,
\begin{equation}
\begin{split}
 f(k,p&,n,L) = \ \left( p \log_2 \frac{n}{n-p} + p \log_2 \frac{n-p}{p}\right)\\
 		&-L \left( n \log_2 \frac{n}{n-k} - p \log_2 \frac{p}{p-k} + k \log_2 \frac{n-k}{p-k} \right).
\end{split} 
\end{equation}

As (\ref{eq:R}) depends on many parameters, it is hard to figure out the reduction in the number of subspaces form $\mathcal{R}$. To demonstrate this better, we can fix $\delta = \frac{p}{n}$, $\beta = \frac{k}{p}$ and $t = \frac{L}{n}$, and plot $f$ based on $n$, which is an approximation for $\mathcal{R}$, and showing the bounds of (\ref{eq:R}) with some error bars. If we choose $\delta = \frac{1}{4}$, $\beta = \frac{1}{10}$ and $t = 100$, the bounds for $\mathcal{R}$ are plotted as functions of $n$ in Figure \ref{numberofsubspaces}. As $2^\mathcal{R}$ is the ratio between the number of subspaces in the new model and $k$-sparsity model, we can see that ratio is significantly reduced for large $n$. In other words, the search space for the solution is now much smaller, which may boost the exact recovery using practical recovery algorithms, as we can see in the simulation section.

\section{A Practical Optimisation Algorithm} \label{sec:dicoselection}

Although the objective of (\ref{eq:odsformulation}) is quadratic, the optimisation of (\ref{eq:odsformulation}) subject to the non-convex constraints $\mathcal{K}$ and $\mathcal{P}$, is not easy. Most of the efficient optimisation techniques can not be used in this setting. A powerful technique, called the projected gradient, can be used when the projection onto the admissible set is available. In the space of real matrices $\mathbb{R}^{n \times L}$, the projection of a point $\mathbf{X} \in \mathbb{R}^{n \times L}$ onto a closed set $\kappa \subseteq \mathbb{R}^{n \times L}$ is defined by $\mathcal{P}_{\kappa}(\mathbf{X}) := \arg\min_{\mathbf{\Theta} \in \kappa} \|\mathbf{\Theta} - \mathbf{X} \|_{\chi}$, where $\|\cdot\|_\chi$ is the norm of the proposed space. We use the Hilbert - Schmidt, or Frobenius, norm here, as it is more related to the quadratic objective (\ref{eq:nsparse}), \textit{i.e.} using the same normed space, and we can analytically find the projection. In this setting, a projection onto $\
mathcal{K}$ can be found by keeping the $k$ largest coefficients of each column and letting the others be zero. The projection onto $\mathcal{P}$ can be found by keeping the $p$ rows of $\mathbf{X}$ with the largest maximum absolute values and letting the other rows be zero. Sadly, the projection onto the intersection of $\mathcal{K}$ and $\mathcal{P}$ is not analytically possible, the projected gradient algorithm can not be used in its canonical form. A property of the admissible sets $\mathcal{K}$ and $\mathcal{P}$ is that the consequent projections of a point in these sets provide a point in the intersection of them, which may indeed not necessarily be the projection onto $\mathcal{K} \cap \mathcal{P}$. The following lemma shows that alternating projection onto $\mathcal{K}$ and $\mathcal{P}$ converges in a single alternating projections.

\begin{lem}
Let $\X$ be a bounded matrix in $\R^{n \times L}$. The following two statements hold,
\begin{equation}\label{eq:projection}
\begin{split}
 \P_{\mathcal{P}} \P_{\mathcal{K}} (\X) \in& \  \mathcal{P} \cap \mathcal{K} \\
 \P_{\mathcal{K}} \P_{\mathcal{P}} (\X) \in& \  \mathcal{P} \cap \mathcal{K}.
\end{split}
\end{equation}
\end{lem}
\begin{proof}
Projections of $\X$ onto $\mathcal{K}$ or $\mathcal{P}$ shrinks some of $\X$'s non-zero elements to zero and does not produce any further non-zero elements. This simply shows that the projection of a point in $\mathcal{K}$, onto $\mathcal{P}$, gives a new point which is still in $\mathcal{K}$. It assures the first statement. The second statement can be shown similarly.
\end{proof}
\begin{remark}
 The sets $\mathcal{K}$ and $\mathcal{P}$ are non-convex and the projection onto each of these sets may thus be non-unique. In this case we can randomly choose one of the projections.
\end{remark}

\subsection{Proposed Dictionary Selection Algorithm} 

We use a gradient based method which iteratively updates the current solution $\X^{[n]}$, in the negative gradient direction and maps onto a point in $\mathcal{K} \cap \mathcal{P}$, to approximately solve (\ref{eq:odsformulation}). If $\psi(\mathbf{X}) := \|\mathbf{Y} - \mathbf{\Phi} \mathbf{X}\|_{F}^2$, the gradient of $\psi$ can be found as follows,

\begin{equation}
\mathbf{G} := \frac{\partial}{\partial \mathbf{X}} \psi(\mathbf{X}) = 2 \mathbf{\Phi}^H \left(\mathbf{\Phi} \mathbf{X} - \mathbf{Y} \right).
\end{equation}

An important part of the gradient descent methods, is how to select the step size. An efficient step size selection technique for unconstrained quadratic minimisation problems, with objectives like $\psi(\mathbf{X})$, is to use half of the spectral radius of linear operator, here $\mathbf{\Phi}$, as follows,

\begin{equation*}
\mu = \frac{1}{2}\frac{\mathbf{G}^H\mathbf{\Phi}^H\mathbf{\Phi}\mathbf{G}}{\mathbf{G}^H\mathbf{G}}
\end{equation*}

Such a step size is optimal for the first order gradient descent minimisation of the unconstrained problem with the quadratic objective $\psi(\mathbf{X})$. In a constrained minimisation scenario, we can choose a similar initial step size and shrink the size, if the objective increases. It thus needs an extra step to check that the objective is actually not increased after each update of the parameters.
A more clever initial step size was selected in \cite{Blumensath10} for the sparse approximations of $k$-sparse signals. If the support of sparse coefficient vectors are fixed, \textit{i.e.} the overall projection steps do not change the support, the update is only in the direction which changes current non-zero coefficients. When the problem size is shrunk to the space of current support, the problem is quadratic and the step size can be similarly calculated using the gradient matrix $\mathbf{G}$, constrained to the support, as follows,

\begin{equation*}
\mu = \frac{1}{2}\frac{\mathbf{G}_{S}^H\mathbf{\Phi}^H\mathbf{\Phi}\mathbf{G}_{S}}{\mathbf{G}_{S}^H\mathbf{G}_{S}}
\end{equation*}
where $\mathbf{G}_{S} \in \mathbb{R}^{m \times L}$ is the gradient matrix $\mathbf{G}$ masked by the support of $\mathbf{X}$, $S$, as follows,

\begin{equation*}
 \{\mathbf{G}_{S}\}_{i,j} =  \begin{cases}
				\{\mathbf{G}\}_{i,j} & \{\mathbf{X}\}_{i,j} \ne 0 \\
				0	& 	\operatorname{Otherwise}
                             \end{cases}
\end{equation*}

A pseudo-code for the algorithm is presented in Algorithm \ref{alg:projGrad}. The condition which is checked in line \ref{algline:majorcheck}, guarantees that the algorithm reduces the objective by updating the coefficients. As the dictionary selection algorithm \ref{alg:projGrad} is based upon a gradient projection type technique, the learned dictionary may be more suitable for such greedy sparse approximation techniques. However, the simulation results show that the reference dictionary can be recovered using this algorithm, given a rich set of training samples. If the real signal is sparse and the dictionary satisfies the exact recovery conditions, the dictionary is thus optimal for any sparse recovery algorithms.
% As the constraint set $\mathcal{K} \cup \mathcal{P}$ is unbounded\footnote{For any given bounded value $t$, there exists at least a point $\mathbf{X} \in \mathcal{K} \cup \mathcal{P}$ such that $\|\mathbf{X}\|_F > t$.}, it is necessary to check the stability of the algorithm. 
% The following theorem proves the stablility of Algorithm \ref{alg:projGrad}, in a Lyapunov sense.

In the following theorem, we prove that Algorithm \ref{alg:projGrad} is numerically stable and the generated sequence has limit points.

\begin{theorem}
 Let $\X^{[0]} \in \R^{m \times L}$ be a bounded initial point. The gradient based method of Algorithm \ref{alg:projGrad}, generates a bounded sequence of solutions, which accumulate.
\end{theorem}
\begin{proof}
 As the algorithm reduces the objective at each iteration, the search space is a bounded subset of $\mathcal{K} \cap \mathcal{P}$, based upon Corollary \ref{corol:bounded}. $\mathcal{K} \cap \mathcal{P}$ is a closed set, the search space is then a compact subset of $\R^{m \times L}$. The sequence generated by Algorithm \ref{alg:projGrad}, lives in a compact set, which is enough to have bounded accumulation points, based on the Bolzano-Weierstrass theorem. 
\end{proof}

\begin{algorithm}[t]
\caption{ Alternating Projected Gradient for Dictionary Selection}\label{alg:projGrad}
\begin{algorithmic}[1]
 \STATE \textbf{initialisation:} $\mathbf{X}^{[0]}$, $S = \operatorname{supp}\left(\mathcal{P}_{\mathcal{K}}\left( \mathcal{P}_{\mathcal{P}}\left( \mathbf{\Phi}^H\mathbf{Y}\right)\right)\right)$, $\rho < 1$, $\beta < 1$, $\epsilon \ll 1$, $t = 0$, $K \ge 1$ and $i = 0$
%  \FOR{$n = 0$ \textbf{to} $K-1$}
 \WHILE {$i < K \ \& \ t \ne 1$}
 \STATE $\mathbf{G} = 2 \mathbf{\Phi}^H \left(\mathbf{\Phi} \mathbf{X}^{[i]} - \mathbf{Y} \right)$
 \STATE $\mu = \frac{1}{2}\frac{\mathbf{G}_{S}^H\mathbf{\Phi}^H\mathbf{\Phi}\mathbf{G}_{S}}{\mathbf{G}_{S}^H\mathbf{G}_{S}}$
 \STATE $\mathbf{Z} = \mathcal{P}_{\mathcal{K}}\left( \mathcal{P}_{\mathcal{P}}\left( \mathbf{X}^{[i]} - \mu \mathbf{G}\right)\right)$
 \IF {$\|\mathbf{X}^{[i]}- \mathbf{Z}\|_{\mathsmaller{F}}^2 < \epsilon$} \STATE $t = 1$ \ENDIF
 \IF {$t \ne 1$}
\WHILE {$\mu > \frac{\rho}{2} \frac{\|\mathbf{X}^{[i]}- \mathbf{Z}\|_{\mathsmaller{F}}^2}{\|\mathbf{\Phi} \left(\mathbf{X}^{[i]}- \mathbf{Z}\right)\|_{\mathsmaller{F}}^2} $} \label{algline:majorcheck}
\STATE $\mu = \beta . \mu$
\STATE $\mathbf{Z} = \mathcal{P}_{\mathcal{K}}\left( \mathcal{P}_{\mathcal{P}}\left( \mathbf{X}^{[i]} - \mu \mathbf{G}\right)\right)$
\ENDWHILE
\ENDIF
\STATE $i = i+1$
\STATE $\mathbf{X}^{[i]} = \mathbf{Z}$
\STATE $S = \operatorname{supp}\left(\mathbf{X}^{[i]}\right)$
 \ENDWHILE
\STATE $\mathbf{X}^* = \mathbf{X}^{[i-1]}$
 \STATE \textbf{output:} $\mathbf{X}^*$
\end{algorithmic}
\end{algorithm}

\begin{figure}[t]
\centering
% \centerline{\epsfig{figure=3KPSparseComp.eps,width=9cm}}
\includegraphics[width=8.5 cm]{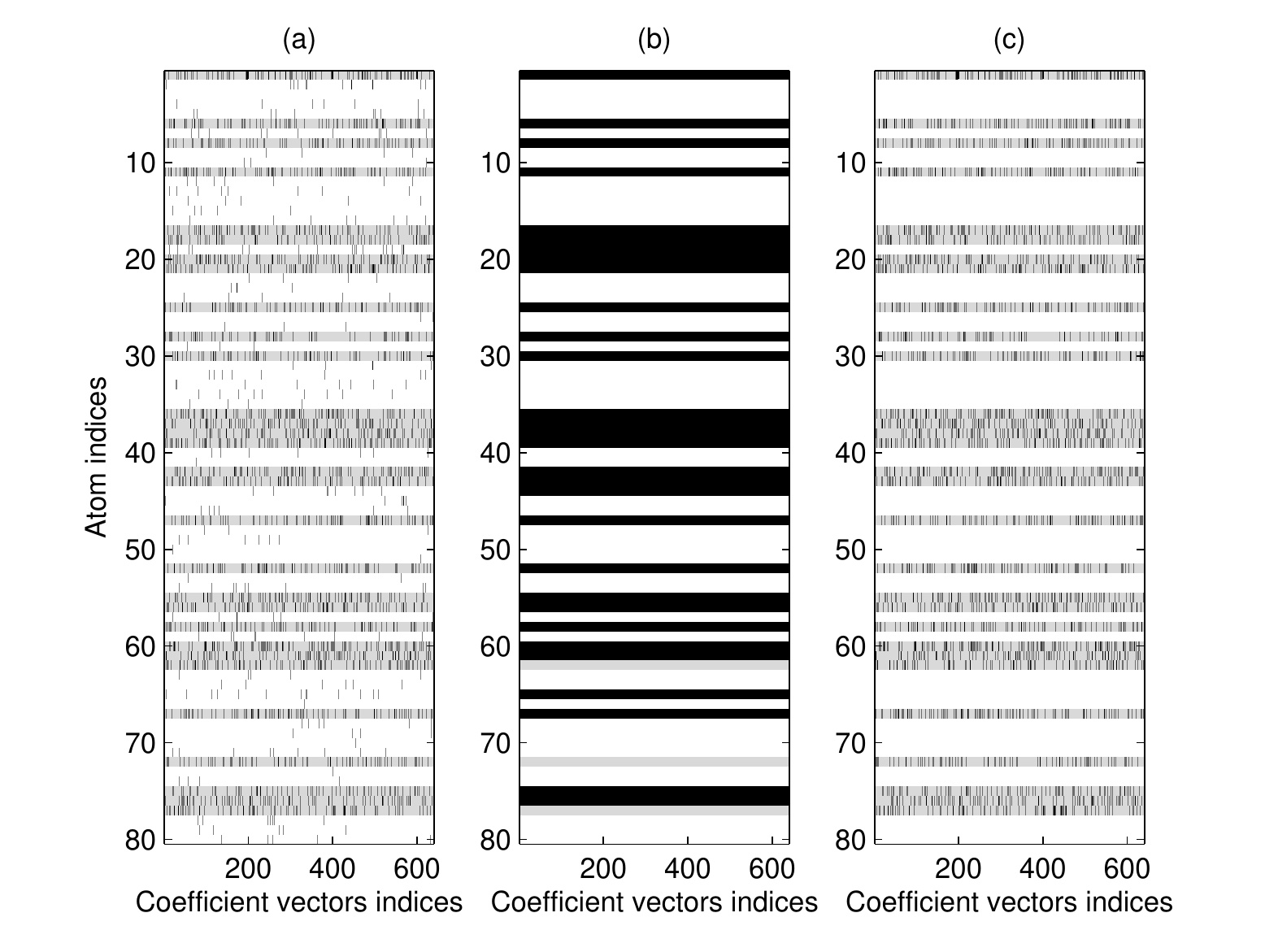}
\caption{Dictionary selection results using, (a) $\mathcal{K}$, (b) $\mathcal{P}$ and (c) $\mathcal{K}\cap \mathcal{P}$ as admissible sets. The black dots in each plot indicate non-zero coefficients. In plot (b), as dots are very populated, we observe solid horizontal lines. Gray horizontal lines are plotted as a guideline, for the correct dictionary.}
\label{3KPSparseComp}
\end{figure}

\begin{figure}[t]
\centering
% \centerline{\epsfig{figure=phaseTransition.eps,width=10cm, height=3.5cm}}
\includegraphics[width=8.5 cm]{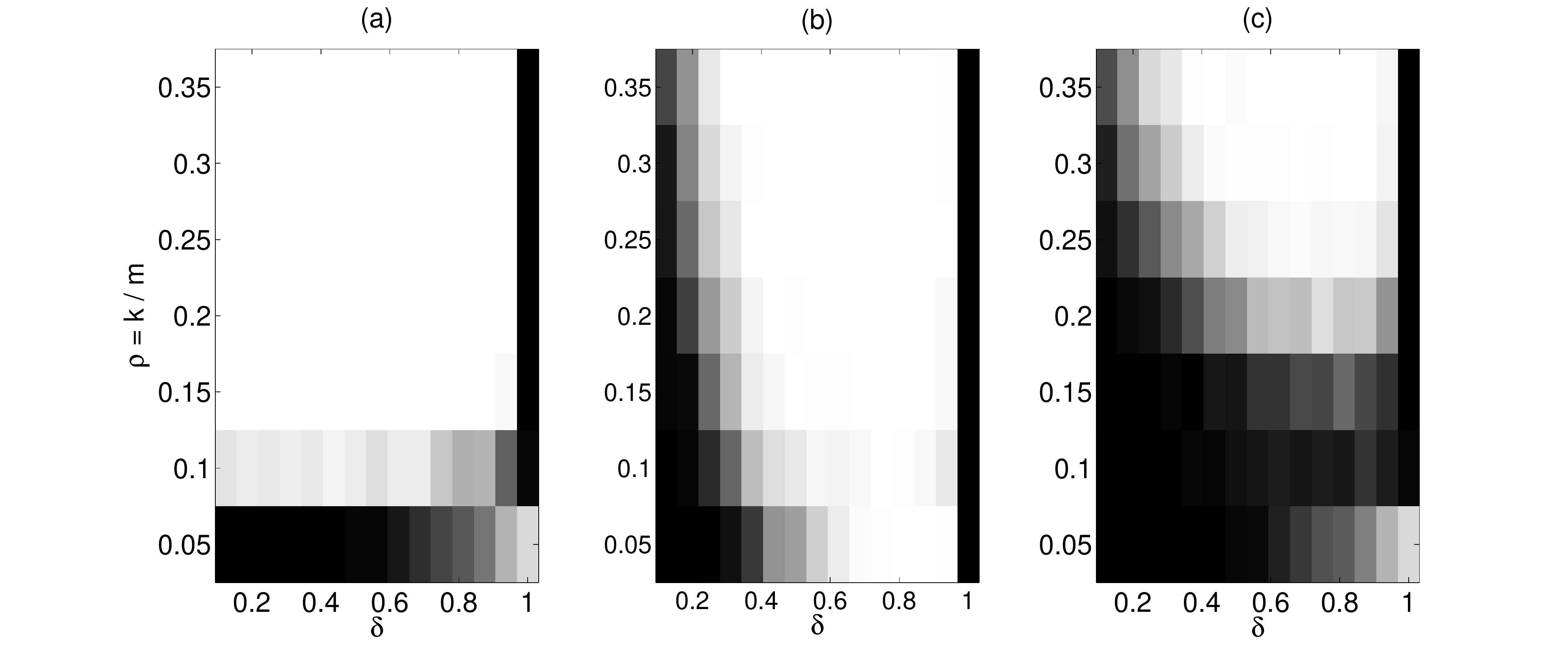}
\caption{Phase transition using, (a) $\mathcal{K}$, (b) $\mathcal{P}$ and (c) $\mathcal{K}\cap \mathcal{P}$ as admissible sets. The black area indicates successful recovery of the dictionary.}
\label{pt3}
\end{figure}

\section{Simulations} \label{sec:simulations}

In the first experiment, a dictionary $\mathbf{\Phi} \in \mathbb{R}^{20 \times 80}$ was randomly generated using a normal zero mean distribution with unit variance and normalised to have unit $\ell_2$-norm on each column. The target dictionary $\mathbf{D} \in \mathbb{R}^{20 \times 30}$ was generated by randomly selecting $p = 30$ atoms of $\mathbf{\Phi}$.
% $L = 320$ $k$-sparse coefficient vectors, with $k =4$, were generated by randomly selecting the support and the coefficients had uniformly random magnitudes between 0.2 and 1.
A number $L = 320$ of $k$-sparse coefficient vectors (with $k = 4$), were generated by randomly selecting the support, with a uniform distribution of the magnitudes in $[0.2,1]$ and random signs. 
A set of training matrix $\mathbf{Y}$ of length $L$ were generated using the generative model and randomly generated sparse vectors. To recover the reference dictionary $\mathbf{D}$, given $\mathbf{Y}$, $p$ and $k$, we used a gradient descent based algorithm similar to Algorithm \ref{alg:projGrad}, with three different admissible sets, and demonstrate the superiority of the proposed technique. We first used $\mathcal{K}$ from (\ref{eq:ksparse}) and no constraint on the row-sparsity of the coefficient matrix $\mathbf{X}$ and showed the recovered support of the sparse matrix in the left panel (a) of Figure \ref{3KPSparseComp}. If we only assume joint sparsity model and use $\mathcal{P}$ from (\ref{eq:nsparse}) as the admissible set, we find the coefficient matrix whose support is shown in the middle panel (b) of the same figure. Using both constraint sets, as explained in Algorithm \ref{alg:projGrad}, provides a coefficient matrix whose support is shown in the right panel. The correct $\mathcal{J}$ is 
shown in these plots using some \emph{grey 
lines}. It is clear that the proposed projected gradient onto both sets can correctly recover $\mathcal{J}$, where the other two methods have some errors in the recovery.

This experiment can be repeated for different $\delta = \frac{p}{n}$ and $\rho = \frac{k}{m}$ by selection a range of $p$ and $k$'s, while keeping $m$ and $n$ fixed. If we repeat the simulations 100 times for each setting and calculate the average exact dictionary recovery, we can plot the phase transition for each methods. We have plotted such phase transitions in Figure \ref{pt3}, with $k$ sparsity constraint in (a), $p$ joint sparsity constraint in (b) and proposed constraint in (C). The black colour means high exact dictionary recovery. The area with exact recovery in (c) is larger than the same areas in (a) and (b) added together. This clearly demonstrates the relevance of the new framework.
 
In the next set of experiments, we will select a subset of the Curvelet \cite{Candes06c} dictionary for the sparse representation of fingerprints. We chose a Curvelet transform for the image size $64$ by $64$. The mother dictionary $\mathbf{\Phi} \in \R^{4096 \times 10521}$ is roughly $2.59$ times overcomplete, which we want to shrink to half size, \textit{i.e.} $\D \in \R^{4096 \times 5260}$. This is indeed a large scale dictionary learning problem, which is difficult to solve in a standard dictionary learning setting. With the help of the proposed method, we can handle such a big dictionary selection process, as we need fewer training samples, only need to keep a sparse matrix, \textit{i.e.} sparse representation matrix, in the memory and use the fact that the mother dictionary has a fast implementation. We assume the sparsity of each image patch is $k = 1052 \approx 0.1 N$ and $L = 64$. We used two different settings here to choose the dictionary, a) $p$-joint sparsity model and b) $(k,p)$-overcomplete 
joint sparse model. The simulations were done in the Matlab environment, on a 12-core, 2.6 GHz linux machine, which respectively took 72 and 90 seconds to learn $\D_p$ and $\D_{(k,p)}$. Another fingerprint image was used to test the selected dictionaries. The original image and the $k$ sparse representation of the original image with $\mathbf{\Phi}$ are shown in the first row of Figure \ref{FingerCurves4}. The $k$-sparse representation with the learned $\D$'s are shown in the second row of this figure. The left image is the representation with learned $\D_{p}$, when the model was $p$-joint sparse and the right image is the same, but with $\D_{(k,p)}$, where the $(k,p)$-overcomplete joint sparse model was incorporated. As we can see the PSNR of the representation with the shrunk dictionary $\D_{(k,p)}$ is slightly better than the other. We can also see the bottom-right quarter of these images in Figure \ref{FingerCurvesComp4}, in the same order. The sparse reconstructed images are actually denoised and the 
reconstructed image using $\D_{(k,p)}$ is more similar than $\D_{p}$ to the image reconstructed using $\mathbf{\Phi}$. 

We setup a new experiment with audio signals to demonstrate the performance of the proposed dictionary selection algorithm in comparison with the fixed dictionaries and another dictionary learning method. To this end, we used some recorded audio data from BBC radio 3 (mostly classical musics), and down-sampled the signals at a sampling rate of 32 kHz, as there is very little energy above 32 kHz. We randomly selected a $\Y \in \R^{1025 \times 8196}$ from more than eight hours of recorded audio. A three times overcomplete mother dictionary was generated using a two times frequency oversampled DCT plus the Delta Dirac transform, \textit{i.e.} identity matrix. The reason for such a selection is to incorporate the temporal and harmonic properties of the audio. There has been a question on how useful can be to combine such dictionaries and how many DCT atoms are necessary. We thus found a subset, \textit{i.e.} $p = \frac{3}{2}*1024$, of the mother atoms. If we run the proposed dictionary selection algorithm with $k = 128$, for $K = 1000$ iterations, and plot the frequency of appearance of the mother atoms in $\X$, we get the plot of Figure \ref{AudioDicoIndex}. The low-frequency DCT atoms have been used most, while high-frequency DCT atoms have not been selected in $\D$. Although there is no regular pattern for the selected delta Diracs, it is clear that some Dirac atoms close to the boundary of the window have been selected, \textit{i.e.} close to the atom indices 2048 and 3072. 
% These atoms can compensate the lack of DCT atoms with different phases, as the audio is an ideal dictionary have to be shift invariant, which makes a very large dictionary and we are not interested here. 
If we plot the $\ell_2$ errors of representing a set of test data $\widehat{\Y}$, which is randomly selected from the same audio database, through out the IHT iterations, we get the plots of Figure \ref{AudioSRCompSD}. The $\ell_2$ errors corresponding to using the mother and two times overcomplete DCT dictionaries, are also shown for the reference with solid and dash-doted lines, respectively. The final SNR using the selected dictionary is slightly worse than by using the mother dictionary, but is significantly better than by using the two times overcomplete DCT. For a comparison, we also ran the sparse dictionary learning \cite{Yaghoobi09b}, with the same training data samples. The reason for selecting this dictionary learning algorithm is that it has some similarities with the proposed framework here, and it provides a relatively fast dictionary, \textit{i.e.} an extra sparse matrix-vector multiplication is also necessary. To learn a dictionary for this signal size, using the canonical dictionary learning algorithms, \textit{e.g.} K-SVD, MOD and MMDL \cite{Rubinstein10}, needs many more training samples and the computational time is very high. In the sparse dictionary learning, we used the same mother dictionary we used earlier, the objective multipliers $\lambda = \gamma = .01$ and ran the simulations for 1000 iterations. The $\ell_2$ errors of using the sparse learned dictionary is shown by the dotted line in Figure \ref{AudioSRCompSD}. The final SNR is not as good as when 
we use other dictionaries.

One aim of the proposed dictionary selection method is to provide a fast dictionary. For this reason, we measured the average calculation time of the forward and backward applications of $\D$ on a 2.6GHz Intel Xeon processor machine. The application of $\D$ and $\D^\mathsmaller{T}$ as some fast operators, respectively takes, 230 and 124 ns. If Implementing the same computation as matrix-vector multiplication with the same dictionary and input data, takes 288 and 285 ns, respectively. This shows that, using a fast selected dictionary, speeds up the practical sparse approximation algorithms, as applying the dictionary and its transposed, often are the most computationally expensive parts of such algorithms.

\begin{figure}[t]
\centering
% \centerline{\epsfig{figure=FingerCurves4.eps,width=9cm}}
\includegraphics[width=8.5 cm]{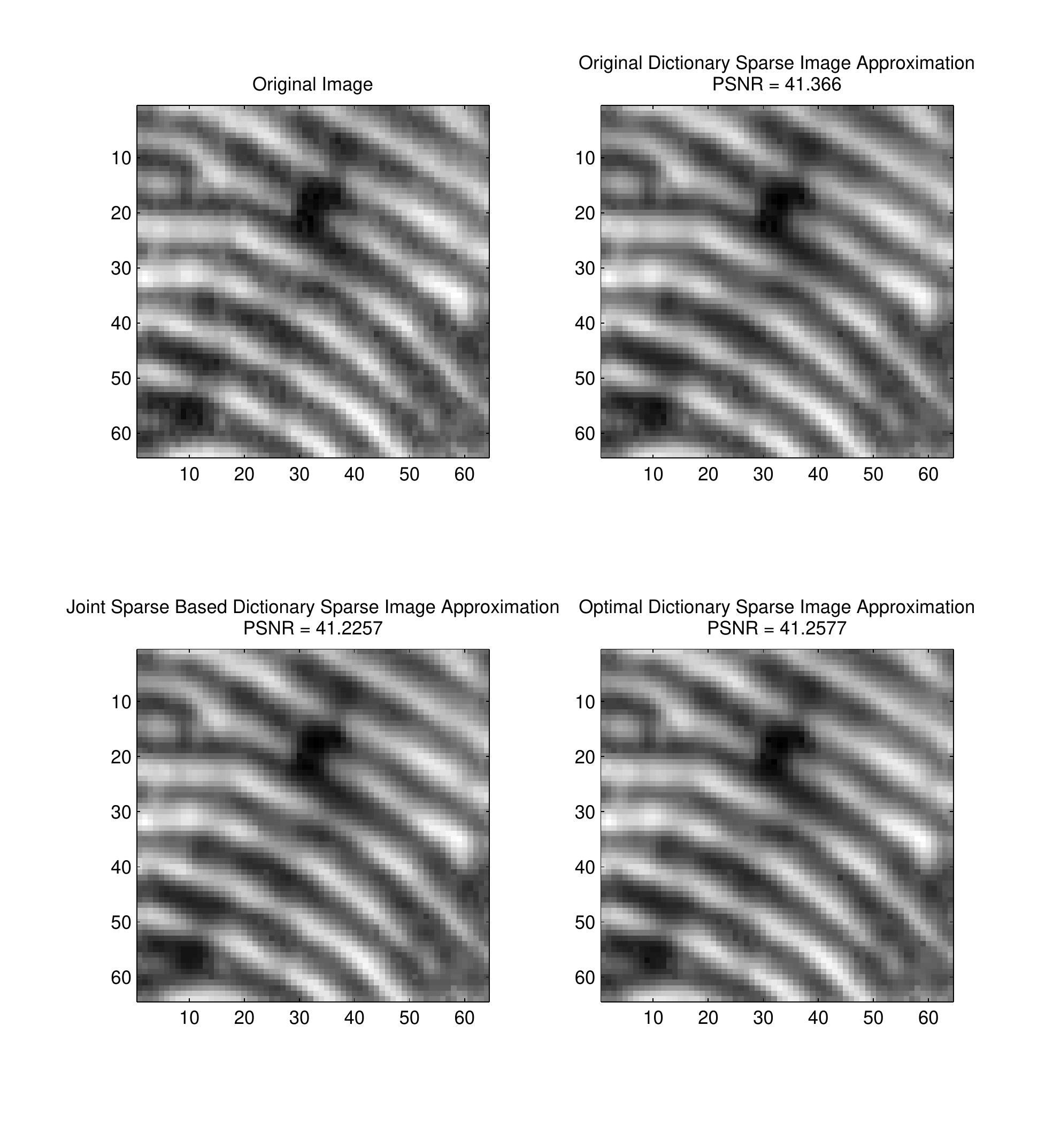}
\caption{The original image (top left), the $k$ sparse representation of the original image with the dictionaries, $\mathbf{\Phi}$ (top right), $\D_{p}$ (bottom left) and $\D_{(k,p)}$ (bottom right).}
\label{FingerCurves4}
\end{figure}

\begin{figure}[t]
\centering
% \centerline{\epsfig{figure=FingerCurvesComp4.eps,width=9cm}}
\includegraphics[width=8.5 cm]{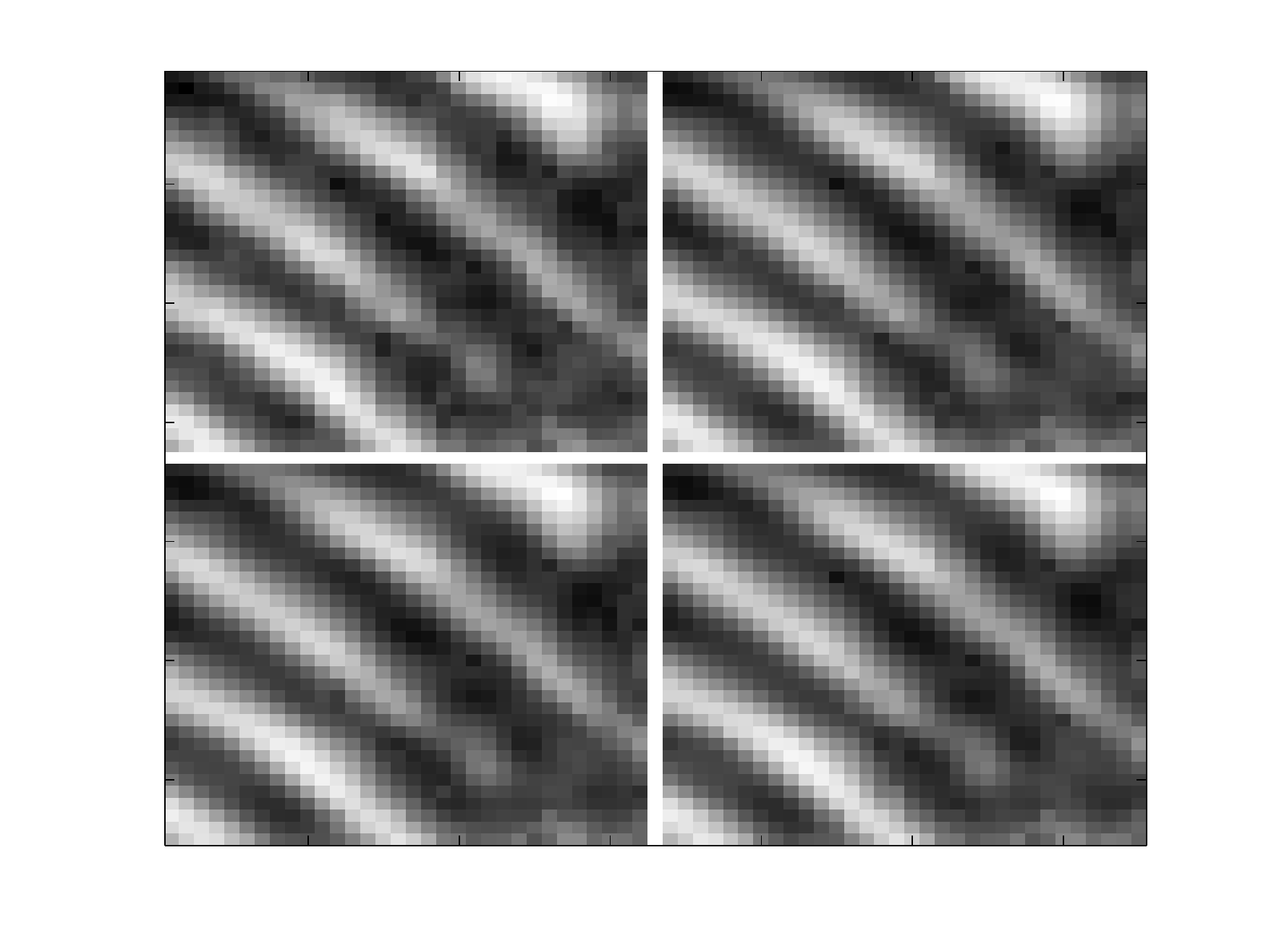}
\caption{The bottom-right quarter of the images shown in Figure \ref{FingerCurves4}, in the same order.}
\label{FingerCurvesComp4}
\end{figure}

\begin{figure}[t]
\centering
\includegraphics[width=8.5 cm]{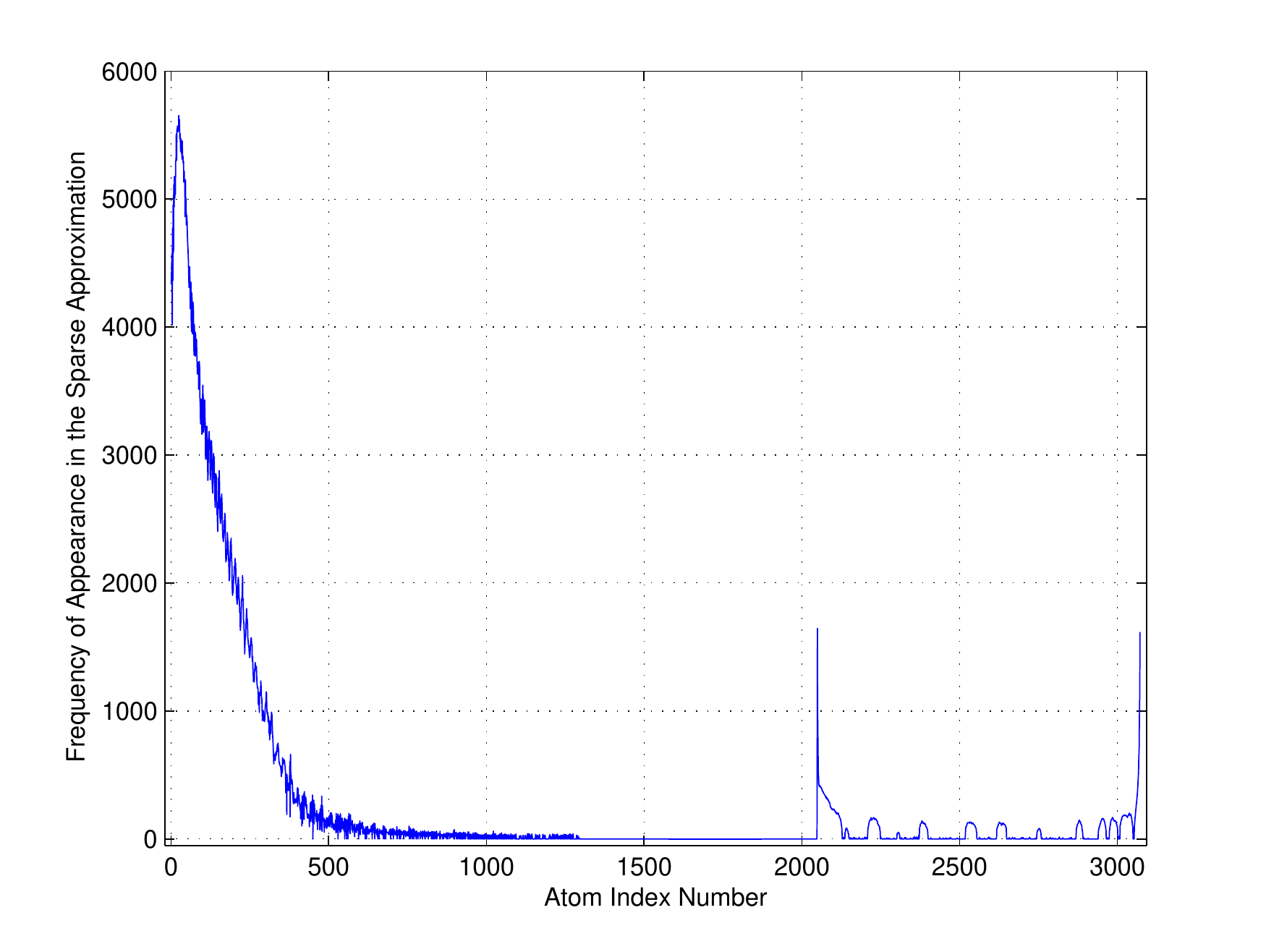}
\caption{The frequency of selected atoms, per 8192 trials. The first 2048 atoms are the two times frequency oversampled DCT and the last 1024 atoms are Dirac functions.}
\label{AudioDicoIndex}
\end{figure}

\begin{figure}[t]
\centering
\includegraphics[width=8.5 cm]{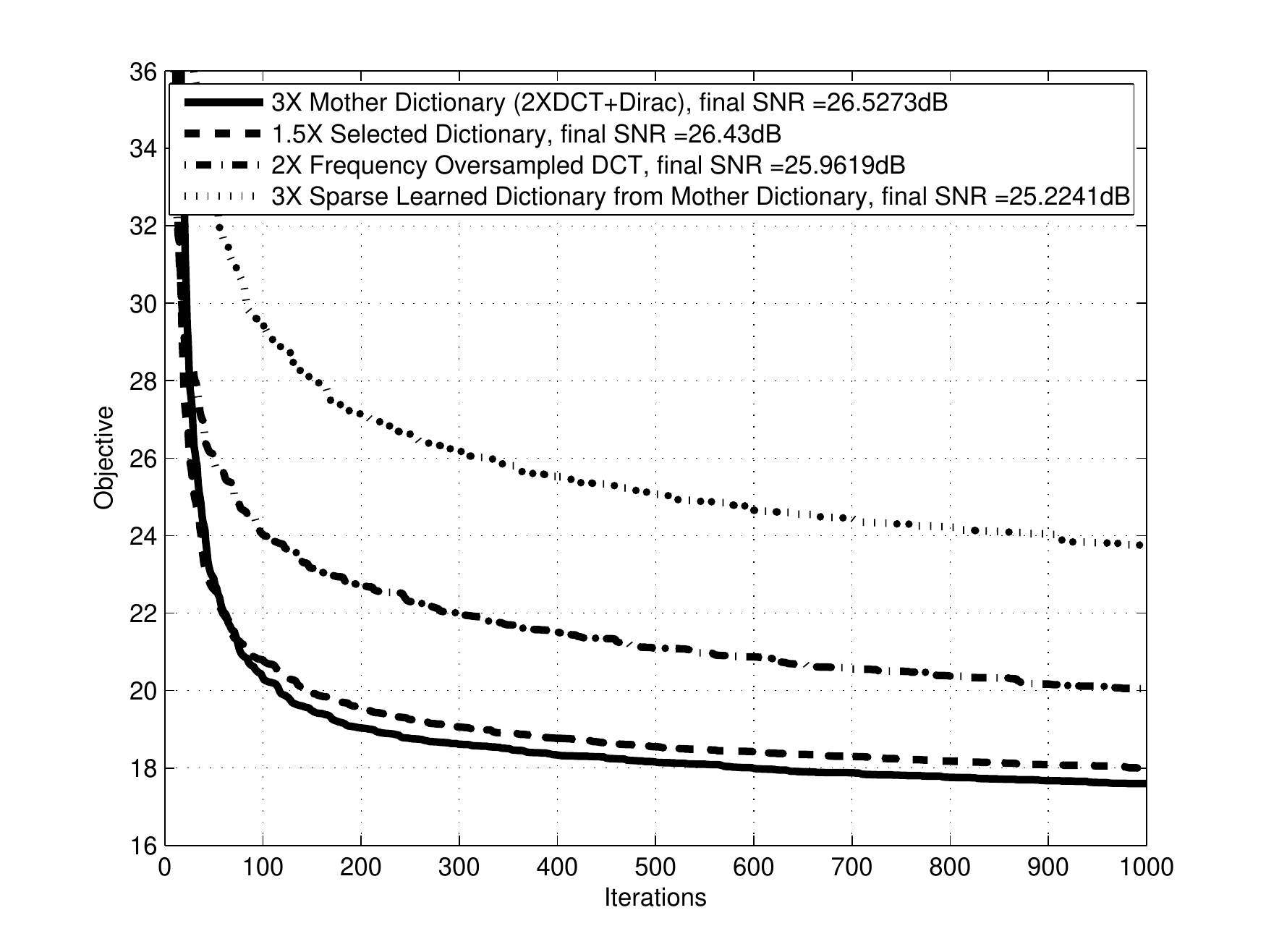}
\caption{The $\ell_2$ norm error of representations of 8192 testing trials, using Normalised IHT and different dictionaries. The dictionaries are: (a) three times overcomplete DCT+Dirac mother dictionary, (b) the one and a half times overcomplete selected dictionary, (c) a two times frequency oversampled DCT and (d) the leaned sparse dictionary using the mother dictionary of (a).}
\label{AudioSRCompSD}
\end{figure}

\section{Summary and Future work} \label{sec:conclusion}

We presented a new technique for dictionary selection for the linear sparse representation, when a collection of possibly suitable atoms and some exemplar signals are available. The dictionary selection problem is reformulated as a more general form of the joint sparse approximation problem, when the number of active locations in sparse coefficients is larger than the size of signal space. As such overcomplete joint sparsity framework has generally infinitely many solutions, the sparsity within the active set helps to regularise the problem. It was shown that the overcomplete joint sparse approximation problem is well-defined under some conditions on the null-space of the matrix generated by the given large set of atoms (mother dictionary). As the objective of the introduced program is continuously differentiable, we used a gradient mapping technique to approximately solve the problem. The introduced algorithm converges in a weak sense (convergence to a bounded non-empty set). 

We presented some synthetic data simulation result to support this hypothesis that the introduced algorithm can recover the original dictionary. The phase plot of the dictionary recovery is compared with two other cases, when we use other sparsity models, namely $k$-sparse and $p$-joint sparse model. As the simulations with synthetic data were promising, we also did some simulations to select a subset of a commonly used dictionary, {\em Curvelet} and Overcomplete DCT+Dirac dictionaries, to reduce the complexity of the sparse coding algorithm. The size of dictionary learning problem is such that it cannot be handled by the vast majority of current dictionary learning algorithms. As we do not need to keep the dictionary in the memory and as the dictionary-vector multiplications can be implemented efficiently, the learning in the new framework is relatively easy. The results show that we can roughly get the same image/audio quality for a specific class of image/audio signals, when we use a smaller dictionary 
than the mother dictionary. 

The new overcomplete joint sparsity model seems an interesting extension of the previously investigated joint sparsity model. We have left the theoretical investigation of exact recovery and other sparse signal processing applications, for the future work.

\bibliographystyle{IEEEtran}
% \bibliography{thesis}
% \bibliography{/home/mehrdad/Projects/Reports/Biblio/biblio.bib}
% \bibliography{/home/s0574225/Project/Reports/Biblio/biblio}
\bibliography{/home/s0574225/Dropbox/Reports/Biblio/biblio.bib}
% \bibliography{/home/mehrdad/Dropbox/Reports/Biblio/biblio.bib}
\end{document}